\documentclass[12pt]{article} 

\title{Uncertainty quantification for distributed regression}
\usepackage{times}

\RequirePackage[OT1]{fontenc}
\RequirePackage{amsmath}

\usepackage{amsmath}
\usepackage{mathtools}

\mathtoolsset{showonlyrefs}
\usepackage{amssymb}
\usepackage{calligra}
\usepackage{amsfonts}
\usepackage[scr=boondox]{mathalfa}

\usepackage{float}
\usepackage{natbib}
\usepackage[strings]{underscore}
\usepackage{hyperref}

\makeatletter
\def\renewtheorem#1{%
	\expandafter\let\csname#1\endcsname\relax
	\expandafter\let\csname c@#1\endcsname\relax
	\gdef\renewtheorem@envname{#1}
	\renewtheorem@secpar
}
\def\renewtheorem@secpar{\@ifnextchar[{\renewtheorem@numberedlike}{\renewtheorem@nonumberedlike}}
\def\renewtheorem@numberedlike[#1]#2{\newtheorem{\renewtheorem@envname}[#1]{#2}}
\def\renewtheorem@nonumberedlike#1{  
	\def\renewtheorem@caption{#1}
	\edef\renewtheorem@nowithin{\noexpand\newtheorem{\renewtheorem@envname}{\renewtheorem@caption}}
	\renewtheorem@thirdpar
}
\def\renewtheorem@thirdpar{\@ifnextchar[{\renewtheorem@within}{\renewtheorem@nowithin}}
\def\renewtheorem@within[#1]{\renewtheorem@nowithin[#1]}
\makeatother

\usepackage{amsthm}

\newtheorem{assumption}{}

\renewtheorem{theorem}{}
\renewtheorem{lemma}{}
\renewtheorem{corollary}{}
\renewtheorem{remark}{}

\author{%
  Valeriy Avanesov \\
  WIAS Berlin\\
  avanesov@wias-berlin.de\\
}

\begin{document}

\renewcommand{\thedefinition}{Definition \thesection.\arabic{definition}}
\renewcommand{\thelemma}{Lemma \arabic{lemma}}
\renewcommand{\thetheorem}{Theorem \arabic{theorem}}
\renewcommand{\theassumption}{Assumption \arabic{assumption}}
\renewcommand{\theremark}{Remark \arabic{remark}}
\renewcommand{\thecorollary}{Corollary \arabic{corollary}}

\maketitle

\begin{abstract}%
	The ever-growing size of the datasets renders well-studied learning techniques, such as Kernel Ridge Regression, inapplicable, posing a serious computational challenge.
	Divide-and-conquer is a common remedy, suggesting to split the dataset into disjoint partitions, obtain the local estimates and average them, it allows to scale-up an otherwise ineffective base approach.
	In the current study we suggest a fully data-driven approach to quantify uncertainty of the averaged estimator.
	Namely, we construct simultaneous element-wise confidence bands for the predictions yielded by the averaged estimator on a given deterministic prediction set.
	The novel approach features rigorous theoretical guaranties for a wide class of base learners with Kernel Ridge regression being a special case.
	As a by-product of our analysis we also obtain a sup-norm consistency result for the divide-and-conquer Kernel Ridge Regression.
	The simulation study supports the theoretical findings.%
\end{abstract}

\newcommand{\A}{\mathcal{A_\rho}}
\newcommand{\tX}{\tilde{X}}
\newcommand{\tXt}{\tilde{X}_t}
\newcommand{\trace}{\mathcal{T(\rho)}}
\newcommand{\Ltwo}{{L}_2\brac{\mathcal{X}, \pi}}
\newcommand{\fhat}{\hat{\mathscr{f}}_n}
\newcommand{\fhatm}{\hat f_p}
\newcommand{\fhatmT}{\hat{\mathsf{f}}_p}
\newcommand{\ftT}{\tilde{\mathsf{f}}_p}
\newcommand{\ftTb}{\tilde{\mathsf{f}}_p^\flat}
\newcommand{\fhatmTb}{\hat{\mathsf{f}}_p^\flat}
\newcommand{\fhatmb}{\hat f_p^\flat}
\newcommand{\fbar}{\bar f}
\newcommand{\fbarT}{\bar{\mathsf{f}}}
\newcommand{\fbarTb}{\bar{\mathsf{f}}^\flat}
\newcommand{\fbarb}{\bar f^\flat}
\newcommand{\Hknorm}[1]{\left\|#1\right\|_{\mathcal{H}_k}}
\newcommand{\HknormSq}[1]{\left\|#1\right\|_{\mathcal{H}_k}^2}
\newcommand{\ftruerhoT}{\mathsf{f}^*_\rho}
\newcommand{\ftrueT}{\mathsf{f}^*}
\newcommand{\Hsnullnorm}[1]{\left\|#1\right\|_{\mathcal{H}^{s_\circ}}}
\newcommand{\lb}{l^\flat_\alpha}
\newcommand{\ub}{u^\flat_\alpha}
\newcommand{\sumi}{\sum_{i=1}^n}
\newcommand{\sumj}{\sum_{j=1}^{\infty}}
\newcommand{\sump}{\sum_{p=1}^P}
\newcommand{\sumt}{\sum_{t=1}^T}
\newcommand{\hatSigma}{\hat \Sigma}
\newcommand{\Sigmastar}{\Sigma^*}
\newcommand{\snull}{s_\circ}
\newcommand{\Lrho}{L_{\rho}}
\newcommand{\normSq}[1]{\left\|#1\right\|^2}
\newcommand{\norm}[1]{\left\|#1\right\|}
\newcommand{\normF}[1]{\left\|#1\right\|_F}
\newcommand{\infnorm}[1]{\left\|#1\right\|_{\infty}}
\newcommand{\tinfnorm}[1]{\left\|#1\right\|_{\widetilde{\infty}}}
\newcommand{\normOne}[1]{\left\|#1\right\|_1}
\newcommand{\fstar}{f^*}
\newcommand{\frho}{f^*_{\rho}}
\newcommand{\fstarbold}{\boldsymbol{{f}}^*}
\newcommand{\R}{\mathbb{R}}
\newcommand{\gboldm}{\boldsymbol{{g}}(p)}
\newcommand{\gboldbm}{\boldsymbol{{g}}^\flat(p)}
\newcommand{\brac}[1]{\left(#1\right)}
\newcommand{\cbrac}[1]{\left\{#1\right\}}
\renewcommand{\H}{\mathcal{H}_k}
\newcommand{\Hsnull}{\mathcal{H}^{s_\circ}}
\newcommand{\Hs}{\mathcal{H}^{s}}
\newcommand{\f}{\mathbf{f}}
\newcommand{\fhatbold}{\mathbf{\hat f}}
\newcommand{\fhatboldm}{\mathbf{\hat f}(p)}
\newcommand{\fhatboldmb}{\mathbf{\hat f}^\flat(p)}
\newcommand{\fbarbold}{\mathbf{\bar f}}
\newcommand{\fbarboldb}{\mathbf{\bar f}^\flat}
\newcommand{\ftruebold}{\mathbf{f^*}}
\newcommand{\ftrueboldrho}{\mathbf{f}^*_{\rho}}
\newcommand{\ftruerho}{f^*_{\rho}}
\newcommand{\dotprod}[2]{\langle#1,#2\rangle}
\newcommand{\param}{\boldsymbol{\theta}}
\newcommand{\gdelta}{\boldsymbol{\gamma}_{\delta}}
\newcommand{\ghat}{\boldsymbol{\hat \gamma}}
\newcommand{\gbar}{\boldsymbol{\bar \gamma}}
\newcommand{\ggamma}{\boldsymbol{\gamma}}
\newcommand{\ggammab}{\boldsymbol{\gamma}^\flat}
\newcommand{\trueparam}{\boldsymbol{\theta^*}}
\newcommand{\trueparamrho}{\boldsymbol{\theta^*_\rho}}
\newcommand{\hatparam}{\boldsymbol{\hat \theta}}
\newcommand{\hatparamm}{\boldsymbol{\hat \theta}_p}
\newcommand{\hatparammb}{\boldsymbol{\hat \theta}_p^\flat}
\newcommand{\E}[1]{\mathbb{E}\left[#1\right]}
\newcommand{\empE}[1]{\mathbb{E_D}\left[#1\right]}
\newcommand{\Eb}[1]{\mathbb{E}^{\flat}\left[#1\right]}
\newcommand{\Var}[1]{\mathrm{Var}\left[#1\right]}
\newcommand{\Varb}[1]{\mathrm{Var}^\flat\left[#1\right]}
\newcommand{\score}{\nabla \zeta}
\newcommand{\Psit}{\Psi^T}
\newcommand{\Psimt}{\Psi^T_p}
\newcommand{\targetPsi}{\tilde{\Psi}}
\newcommand{\targetK}{\tilde{K}}
\newcommand{\Psim}{\Psi_p}
\newcommand{\eps}{\boldsymbol{\varepsilon}}
\newcommand{\epst}{\boldsymbol{\varepsilon}^T}
\newcommand{\epsm}{\boldsymbol{\varepsilon}(p)}
\newcommand{\epsmt}{\boldsymbol{\varepsilon}(p)^T}
\newcommand{\Drho}{D_\rho}
\newcommand{\DrhoSq}{D_\rho^2}
\newcommand{\DrhoSqInv}{D_\rho^{-2}}
\newcommand{\Drhom}{D_\rho(p)}
\newcommand{\DrhomSq}{D_\rho^2(p)}
\newcommand{\DrhomSqInv}{D_\rho^{-2}(p)}
\newcommand{\Prob}[1]{\mathbb{P} \left\{#1\right\}}
\newcommand{\Probb}[1]{\mathbb{P}^\flat \left\{#1\right\}}
\newcommand{\DrhoInv}{D_\rho^{-1}}
\newcommand{\DrhomInv}{D_\rho^{-1}(p)}
\newcommand{\barparam}{\boldsymbol{\bar \theta}}
\newcommand{\inv}[1]{#1^{-1}}
\newcommand{\MInv}{\inv{M}}
\newcommand{\phiinf}{C_\phi}
\newcommand{\C}{{\mathrm{C}}}
\let\oldc=\c
\renewcommand{\c}{{\mathrm{c}}}
\newcommand{\Rclt}{\mathrm{R_{CLT}}}
\newcommand{\Rb}{\mathrm{R^{\flat}}}
\newcommand{\tr}[1]{\mathrm{tr}\brac{#1}}
\renewcommand{\det}[1]{\mathrm{det}\brac{#1}}
\newcommand{\g}{\mathfrak{g}}
\newcommand{\x}{\mathrm{x}}
\newcommand{\p}{\mathfrak{p}}
\newcommand{\s}{{\sigma_{\rho}}(S)}
\renewcommand{\t}{\mathrm{t}}
\renewcommand{\u}{\mathrm{y}}
\renewcommand{\r}{\mathrm{r}}
\newcommand{\X}{\mathcal{X}}
\newcommand{\matern}{Matérn}
\newcommand{\kdd}{k(\cdot, \cdot)}
\newcommand{\abs}[1]{\left|#1\right|}
\newcommand{\minusroot}{^{-1/2}}
\newcommand{\proot}{^{1/2}}
\newcommand{\Frho}{\brac{M^2 + \rho I}}
\newcommand{\indicator}[1]{\mathbb{I}[#1]}
\newcommand{\given}{\middle|}
\newcommand{\Zpt}{Z_{p,t}}
\newcommand{\varb}{\mathcal{G}_\rho(S)}
\newcommand{\Normal}[2]{\mathcal{N}\brac{#1,#2}}
\newcommand{\Vn}{\mathcal{V}_{\rho}(n)}
\newcommand{\VN}{\mathcal{V}_{\rho}(N)}
\newcommand{\VS}{\mathcal{V}_{\rho}(S)}
\newcommand{\B}{\mathcal{B}_\rho}
\newcommand{\DP}{\Delta\mathbb{P}}
\newcommand{\Zbp}{Z^\flat_p}

\section{Introduction}\label{secintro}
Consider a training sample of $N$ i.i.d. response-covariate pairs $(y_i, X_i)$ from $\R\times \X$ such that
\begin{equation}
  y_i = \fstar(X_i) + \varepsilon_i
\end{equation}
for a compact $\X\subseteq \R^d$, and noise  $\varepsilon_i$, which is centered conditional on $X_i$.
That is, we do not require independence of $X_i$ and $\varepsilon_i$, allowing for heteroscedasticity.
Throughout the paper we presume $X_i$ is distributed w.r.t. some continuous measure $\pi$.
Also, let a deterministic prediction set $\cbrac{\tXt}_{t \in [T]}~\subset~\X$ of size $T$ be also given. 
We wish to construct simultaneous confidence intervals, covering the elements of the vector
\begin{equation}
  \ftrueT \coloneqq \brac{\fstar(\tX_1), \fstar(\tX_2), ..., \fstar(\tX_T) }^\top
\end{equation}
with a predefined probability $\alpha$.
That is, we construct element-wise credible bounds for the vector $\ftrueT$ of values of the unknown function $\fstar$.

The problem can be naturally solved in a Bayesian setting with Gaussian Process Regression (GPR) being a popular and well studied option  \citep{RW,bilionis2012multi,he2011single,chen2013gaussian}.
Having imposed a prior over the true function $\fstar$ and presumed a distribution of noise, one can obtain a posterior distribution over $\fstar$ and construct a Bayesian credible set over $\ftrueT$. Properties of the GPR posterior are well understood \citep{vaart2011information,van2009adaptive,Bhattacharya2017}.
In contrast, the current study follows a fully frequentist line of thought.
That is, no prior is imposed and existence of a single true function $\fstar$ is presumed.
Interest for the problem in a frequentist setting has been rising recently. \cite{Yang2017} suggest to use GPR posterior to construct the sup-norm confidence bands relying on an asymptotic distribution. Fully data-driven approaches have been developed on top of Kernel Density Estimation \citep{Hall2013,Cheng2019}.

Another issue we address in this paper is the high computational cost of some regression approaches (e.g. GPR and Kernel Ridge Regression (KRR), yielding the same point-estimate), rendering them inapplicable to the modern-day datasets.
To that end, among other options, mainly exploiting low-rank approximations \citep{scholkopf1998nonlinear,williams2001using,Rudi2015,bach2013sharp,fine2001efficient}, a divide-and-conquer algorithm may be employed.
Assume, we are given a base algorithm $\A$ (for $\rho$ being a hyperparameter), which, given a sample $\{(y_i, X_i)\}_{i=1}^n$ of an arbitrary size $n$, yields an estimator $\fhat$ of the true function $\fstar$.
The suggestion is to split the dataset into $P$ disjoint partitions and run the algorithm independently for each of them, thus obtaining $P$ estimators $\fhatm$ and eventually averaging them.
Not only does this allow for a trivial distributed implementation, but also drastically cuts the CPU-time if the computational complexity of $\A$ is super-linear.
E. g. for GPR (and KRR) time complexity is $O(N^3)$, while the use of the divide-and-conquer technique allows to bring it down to $O(N^3/P^2)$.
The strategy features the consistency guarantees for parametric $\A$ of moderate dimensionality $d(n)/n \rightarrow \c\in(0,1)$ \citep{Rosenblatt2016}.
\cite{Deisenroth2015} have sped-up GPR using divide-and-conquer.
In \citep{Zhang2015} the Fast-KRR approach has been suggested, being the divide-and-conquer strategy, applied to KRR as the base algorithm. The properties of the point-estimate have been extensively studied \citep{szabo2015frequentist,Lin2017,Mucke2018}.
\cite{avanesov2020} has suggested a bootstrap approach, quantifying uncertainty of the  Fast-KRR estimator in terms of $\Ltwo$-norm, which is of little use for the problem at hand as long as $\tXt$ are not drawn from $\pi$ and we wish to construct element-wise bounds.
In the current study we suggest to enrich the divide-and-conquer paradigm with a non-trivial bootstrap procedure (resembling the one suggested in \cite{avanesov2020}), yielding element-wise confidence bands for the averaged prediction.

The further hardship we cope with is heteroscedasticity. GPR presumes $\varepsilon_i$ to be Gaussian and i.i.d., and the credible bands it yields are of limited value under violation of the assumption.
The frequentist study by \cite{Yang2017} relies on the same noise assumptions, while \cite{avanesov2020} drops Gaussianity, still requiring the noise to be i.i.d.
In this paper we do not impose either.

The contribution of the paper consists in enriching the divide-and-conquer paradigm with a novel approach to constructing element-wise confidence bands for heteroscedastic regression.
The method is computationally efficient and embarrassingly parallel.
It does not rely on any kind of asymptotic distribution, being fully data driven.
To the best of the author's knowledge it is the first method of the kind.
Validity of the approach is rigorously justified for a wide class of base algorithms.
The developed theory is applied to Fast-KRR as a motivating example.
As a byproduct of our analysis of KRR, we also obtain a sup-norm high-probability consistency bound for the averaged estimator.
The formal results are supported by a simulation study.

\section{The method}
Consider an algorithm $\A$ producing an estimator $\fhat$ given a sample of an arbitrary size $n$.
Following the divide-and-conquer strategy, we split the set of indices $\{1,2,..,N\}$ into $P$ disjoint sets $\{S_p\}_{p=1}^P$ of size $S\coloneqq\abs{S_p} = N/P$ (we presume $N/P$ is natural for simplicity) and apply $\A$ independently to each of the partitions
\begin{equation}\label{krrm}
  \fhatm = \A\brac{\cbrac{X_i}_{i \in S_p}}.
\end{equation}
Then average the produced estimators
\begin{equation}
  \fbar \coloneqq \inv{P}\sum_{p=1}^P \fhatm.
\end{equation}
Similarly to $\ftrueT$ denote
\begin{equation}
  \fhatmT \coloneqq \brac{\fhatm(\tX_1), \fhatm(\tX_2), ...,
  \fhatm(\tX_T)}^\top
\end{equation}
and its expectation as $\ftruerhoT \coloneqq \E{\fhatmT}$.
In the same way we introduce
\begin{equation}
  \fbarT \coloneqq \brac{\fbar(\tX_1), \fbar(\tX_2), ...,
  \fbar(\tX_T)}^\top.
\end{equation}
Now we wish to quantify the deviation of $\fbarT$ from $\ftrueT$ via providing the bands $l_\alpha, u_\alpha\in \R^T$ such that
\begin{equation}
  \Prob{l_\alpha < \fbarT - \ftrueT < u_\alpha} = 1 - \alpha
\end{equation}
for a given confidence level $\alpha$.
Here and below the signs $>$ and $<$ used for the vectors of the same dimensionality mean element-wise comparison.
The complicated nature of the distribution of $\fbar$ makes it difficult to choose the bands directly.
To that end we suggest to define a bootstrap procedure, resampling the local estimators $\fhatmT$ with replacement $P$ times and averaging them, thus defining the random bootstrap measure $\mathbb{P}^\flat$. Namely,
\begin{equation}\label{fbootdef}
  \fhatmTb \sim U\left[\cbrac{\fhatmT}_{p=1}^P\right] \text{ for all } p \in [P]
\end{equation}
\begin{equation}
  \fbarTb \coloneqq P^{-1}\sump \fhatmTb.
\end{equation}
Finally, choose the bands $\lb, \ub \in \R^T$ s.t. the bootstrap counterpart of the estimator is covered with the nominal probability of $1-\alpha$
\begin{equation}
  \Probb{\lb < \fbarTb - \fbarT < \ub} = 1-\alpha,
\end{equation}
while the chance for each component $\fbarTb_t$ to violate the respective bound is the same
\begin{equation}
  \Probb{\brac{\lb}_t > \fbarTb_t - \fbarT_t} = \Probb{\fbarTb_t - \fbarT_t > \brac{\ub}_t} = \c \in (0,1) \text{ for all } t \in [T].
\end{equation}

\begin{remark}
  As \eqref{fbootdef} suggests, we will draw $P$ times with replacement during the bootstrap stage, thus defining the bootstrap sample to be of the same size as the real one. 
  Choosing a smaller number of $\fhatmTb$ would lead to worse remainder term in \ref{generaltheorem}, while choosing a larger one would not improve it.
\end{remark}
We establish an upper bound for the divergence
\begin{equation}
  \DP \coloneqq \sup_{l,u \in \R^T} \abs{\Prob{l < \fbarT - \ftrueT < u} - \Probb{l <\fbarTb - \fbarT < u}},
\end{equation}
justifying the use of $\lb$ and $\ub$ instead of $l_\alpha$ and $u_\alpha$.
As a motivating example of the base algorithm $\A$ we employ Kernel Ridge Regression (KRR)
as performance of the point-estimate $\fbar$ has been extensively studied for such a choice \citep{Zhang2015,szabo2015frequentist,Lin2017,Mucke2018}. Consider a kernel operator $\kdd$, inducing a RKHS $\H$ endowed with a norm
\begin{equation}\label{rkhsnormdef}
  \norm{f}^2_{\H} = \sumj \frac{\dotprod{f}{\phi_j}^2}{\mu_j},
\end{equation}
where $\phi_j \in L_2(\X, \pi)$ and $\mu_j$ are orthonormal eigenfunctions and eigenvalues of $\kdd$ w.r.t. $\pi$ (their existence is provided by Mercer's theorem). The dot-product is also defined w.r.t $\pi$.
The KRR estimator is obtained based on a dataset $\cbrac{(y_i, X_i)}_{i = 1 }^n$ through the following optimization problem
\begin{equation}\label{krr}
  \fhat \coloneqq \arg\max_{f} \cbrac{-\frac{1}{2n} \sumi \brac{y_i - f(X_i)}^2 - \frac{\rho}{2} \Hknorm{f}^2}.
\end{equation}
The behavior of the estimator has been investigated in great detail \citep{Caponnetto2007,Mendelson2010,van2006empirical,koltchinskii2006local,zhang2005learning}.

\begin{remark}
	The suggested bootstrap procedure is substantially non-trivial.
	Standard bootstrap schemes usually suggest to resample the original data $(y_i, X_i)$, while we resample the values $\fhatmT$ of the local estimators $\fhatm$ instead, which resembles \citep{avanesov2020}.
	Further, in contrast to the classical bootstrap schemes, we do not construct the confidence set for a mean of the objects being bootstrapped $\E{\fhatmT} = \ftruerhoT$, but for the values of the unknown function $\ftrueT$ instead.
\end{remark}

\begin{remark}
    The computational complexity of the bootstrap procedure is $O(P)$, so the overhead it poses is only marginal compared to obtaining the point-estimate.
    A slightly different version of bootstrap may be employed to further improve efficiency of a distributed implementation.
    That is, sample $P$ i.i.d. multipliers $w_p$ s.t.  $\E{w_p} = \Var{w_p} = 1$ and define a multiplier bootstrap counterpart $\fbarT^w \coloneqq P^{-1}\sum_{p \in [P]} w_p \fhatmT$.
    All the results obtained for $\fbarTb$ also hold for $\fbarT^w$, as their first two moments coincide.
\end{remark}

\section{Theoretical analysis}

We open the section with some notation. 
For a function $\norm{\cdot}$ denotes the $L_2(\X, \pi)$-norm, namely $\norm{f}^2 = \int f^2 d\pi$.
For a vector $\norm{\cdot}$ denotes the $\ell_2$-norm. 
$\infnorm{\cdot}$ denotes the sup-norm for functions (the largest absolute value), vectors and matrices (the largest absolute value of an element).
$I$ stands for an identity operator.
We also use $\c$ and $\C$ as generalized positive constants, whose values may differ from line to line and do not depend on $N,P,T$ and $\fstar$. We use $\asymp$ to denote equality up to a multiplicative constant -- namely, $a_i \asymp b_i$ implies $\c b_i \le a_i \le \C b_i$ for all $i$. For two numbers $a$ and $b$ we use $a \vee b$ to denote the maximum of the two.

\subsection{General theory}
The theoretical analysis we develop is fairly general. 
That is, it applies to any learning algorithm $\A$ with hyperparameter $\rho$ (be it uni- or multivariate, discrete or continuous), accepting an arbitrary number $n$ of response-covariate pairs $(y_i, X_i)$ and producing an estimator $\fhat : \X \rightarrow \R$.\footnote{Throughout the paper we formulate the properties of the base algorithm $\A$ w.r.t. a sample of size $n$ which is not to be confused with $N$ being used to denote the size of the sample we apply the divide-and-conquer strategy to.} 
Surely, some assumptions are required.
Specifically, we require control of variance and bias of the algorithm. 
But first, we need to define the mean of the estimator $\ftruerho \coloneqq \E{\fhat} : \X \rightarrow \R$.
\begin{assumption}[Variance control]\label{varianceass}
  Let there be $\Vn$ for the algorithm $\A$ such that for all positive $\x$ and natural $n$
  \begin{equation}
    \Prob{\sup_{t \in [T]}\abs{\brac{\fhat - \ftruerho}(\tXt)} \ge \brac{1+\x}\Vn } \le \exp(-\x).
  \end{equation}
\end{assumption}
\begin{assumption}[Bias control]\label{biasass}
  Let there be $\B$ for the algorithm $\A$ such that for all $\rho>0$ and natural $n$ \begin{equation}
    \sup_{t \in [T]}\abs{\brac{\fstar - \ftruerho}(\tXt)} \le \B.
  \end{equation}
\end{assumption}
We also impose a lower bound on the variance of $\fhatm(\tXt)$
\begin{assumption}\label{lowersass}
  Let there exist a positive $\s$ such that for all $t \in [T]$
  \begin{equation}
    \s \le \brac{\Var{\fhatm(\tXt)}}^{1/2}.
  \end{equation}
\end{assumption}
Also introduce a shorthand notation
\begin{equation}
\varb = \frac{\VS}{\s}.
\end{equation}  
At this point we are ready to formulate the general bootstrap validity result.
\begin{theorem}\label{generaltheorem}
  Let \ref{varianceass}, \ref{biasass} and \ref{lowersass} hold. Also let $P$ be sufficiently large, namely
  \begin{equation}\label{plowerboundth}
    \frac{\varb \log(TP) \log^{2}P }{P} \le 8^{-1}.
  \end{equation}
   Then 
  \begin{equation}
    \begin{split}
      \DP &\le \C {P^{-1/6}} \brac{\varb + \infnorm{\fstar}^{2/3}} \log^{7/6}(TP) \log(P) \brac{\sqrt[3]{\log T}+1} \\
      & +  \frac{\C\sqrt{P}\B \brac{\sqrt{\log T} +1}}{\s}  .
    \end{split}
  \end{equation}
\end{theorem}
The remainder term may seem frustratingly cumbersome at first. Omitting logarithmic terms and assuming the bound in \ref{varianceass} is tight in the sense $\s \asymp \VS$ (which might be hard to guarantee in practice, yet instructive to consider), we conclude, that for the remainder term to be small, it's sufficient to guarantee
\begin{equation}
  \s P^{-1/2} \gg  \B.
\end{equation}
In fact, the condition imposes {\it undersmoothness} (low bias -- high variance regime), which is a common (and in a sense unavoidable) assumption employed to justify validity of the confidence sets \citep{knapik2011bayesian,szabo2015frequentist,ray2017adaptive,Yang2017,avanesov2020}. 
Really, $\s P^{-1/2} \asymp \VS P^{-1/2}$ is effectively an upper bound for the variance of the averaged estimator $\fbar$. 
A need for undersmoothness does indeed forbid the otherwise desired bias-variance trade-off, which suggests to equalize bias and variance. 
The necessary sacrifice can be quantified for a particular choice of $\A$, as we do in Section \ref{KRRsec} for KRR.

The proof strategy is common for bootstrap validity results and relies on Gaussian approximation. The first step is approximation of $\fbarT - \ftruerhoT$ with Gaussian $\gamma$. That is, by the means of CLT by \cite{Chernozhukov2017} we obtain an approximation uniform (just like the rest of approximations in the sketch)  over $l,u \in \R^T$ 
\begin{equation}\label{firststep}
  \Prob{l < \fbarT - \ftruerhoT < u} \approx \Prob{l <  \gamma < u}.
\end{equation}
Using uniformity of the approximation we also have
\begin{equation}
  \Prob{l < \fbarT - \ftrueT  < u} \approx \Prob{l <  \gamma - \brac{\ftrueT - \ftruerhoT} < u }
\end{equation}
and can use the anti-concentration result, namely Nazarov's inequality \citep{nazarov} 
\begin{equation}\label{app1}
  \Prob{l < \fbarT - \ftrueT  < u} \approx \Prob{l <  \gamma < u }.
\end{equation}
In fact, as Nazarov's inequality dictates, for the approximation to hold it is necessary to ensure smallness of bias in comparison to variance, which is exactly where the need for undersmoothness comes from. The next step is to approximate $\fbarTb - \fbarT$ with another Gaussian vector $\gamma^\flat$ 
\begin{equation}\label{app2}
  \Probb{l < \fbarTb - \fbarT  < u} \approx \Probb{l  <  \gamma^\flat < u }.
\end{equation}
Being both centered, $\gamma$ and $\gamma^\flat$ also have similar (in terms of sup-norm) covariance matrices. Empowered with this observation we can apply Gaussian approximation obtained by \cite{Chernozhukov2013}, which yields 
\begin{equation}\label{app3}
  \Prob{l  <  \gamma < u } \approx \Probb{l  <  \gamma^\flat < u }.
\end{equation}
Combining \eqref{app1}, \eqref{app2} and \eqref{app3}, we arrive to $\DP \approx 0$. The strategy is implemented in Section \ref{secgenproof}.

\begin{remark}
  The decay of the remainder term in \ref{generaltheorem} is driven by $P^{-1/6}$, which is mainly due to the high-dimensional central limit theorem by  \cite{Chernozhukov2017} (\ref{cltlemma}). 
  Therein authors hypothesise optimality of the term up to the logarithmic factors.
  
  Another strategy would be to use a more classical toolset, relying on the multivariate Bentkus' CLT \citep{doi:10.1137/S0040585X97981123} and apply Pinsker's inequality for Gaussian comparison. 
  This would make $P$ enter the remainder term with square root instead of the power of $1/6$. 
  Yet due to the low-dimensional nature of the results, the size $T$ of the prediction set will not enter the result only logarithmically anymore. 
  We believe, this proof strategy would produce $\DP = \tilde{O}(\sqrt{T/{P}})$ (up to polylog factors).
  There is another reason for us not to opt for this approach. 
  It would require a lower bound for the smallest eigenvalue of $\Var{\fhatmT}$, which would impose a restriction on the choice of $\tXt$, forbidding, for instance, presence of two similar covariates.
\end{remark}

\subsection{Kernel Ridge Regression as the base algorithm}\label{KRRsec}
First of all, we impose a polynomial rate of decay on the eigenvalues of $\kdd$.
\begin{assumption}[Polynomial eigendecay]\label{polyeigen}
  Let there exist a constant $b>1$ s.t. for  the $j$-th largest eigenvalue $\mu_j$ of $\kdd$
  \begin{equation}
    \mu_j \asymp j^{-b}.
  \end{equation}
\end{assumption}
A prominent example of a kernel which exhibits this behavior is the \matern~kernel.
As demonstrated in \citep{Yang2017}, \ref{polyeigen} holds for it with $b=2\nu+d$,  $\nu$ being its smoothness index. 
Another assumption we impose is uniform boundness of eigenfunctions. \cite{Yang2017a,Bhattacharya2017} prove the assumption holds for \matern~kernel under uniform and normal distributions of covariates on a compact set.
\begin{assumption}[Boundness of eigenfunctions]\label{eigenbound}
  Denote a normalized eigenfunction corresponding to the $j$-th largest eigenvalue as $\phi_j(\cdot)$.
  Let there exist a positive constant $\phiinf$ s.t. \linebreak $\sup_j\infnorm{\phi_j} \le \phiinf $.
\end{assumption}
Now consider an integral operator
\begin{equation}
  W_k(f) = \int_{\X} k(X, \cdot) f(X) d \pi(X)
\end{equation}
and denote {\it effective dimension} as
\begin{equation}
  \trace \coloneqq \tr{{\inv{\brac{W_k + \rho I}}W_k}}.
\end{equation}
Further, we impose a Hölder-type source condition \citep{Mucke2018,Lin2017}
\begin{assumption}\label{rass}
  Let there exist $r\in (1/2,1]$ and $h \in \Ltwo$ s.t. 
  \begin{equation}
    \fstar = W_k^r(h). 
  \end{equation}
\end{assumption}
For $r = 1/2$ this assumption boils down to $\fstar \in \H$, but we have to impose a slightly more restrictive requirement.
In conclusion we impose the sub-Gaussianity assumption over noise conditioned on covariate, that being a common relaxation of Gaussianity.
\begin{assumption}[Sub-Gaussianity]\label{subgaussass}
  Let there exist a constant $\g^2$ s.t. for all $\in \R$ and for all $i \in [N]$ a.s.
  \begin{equation}
    \E{\exp(a\varepsilon_i) \given X_i} \le \exp\brac{\frac{\g^2a^2}{2}}.
  \end{equation}
\end{assumption}
Under these assumptions it can be shown that \ref{varianceass} holds (due to \ref{varianceinf})  with 
\begin{equation}\label{krrvbar}
  \Vn = \C\sqrt{\frac{\trace}{n \rho^{1/b}}} \le \frac{\C}{\rho^{1/b}\sqrt{n}},
\end{equation}
while \ref{biasass} holds (by \ref{biaslemmainf}) with 
\begin{equation}
  \B \le \rho^{r-1/(2b)}\norm{h}.
\end{equation}
Equipped with these results it is relatively straightforward (see section \ref{KRRproofsec}) to establish a high-probability bound for the estimator $\fbar$ in sup-norm if $\A$ is KRR, complementing the various bounds on expected risk obtained in \citep{Zhang2015,szabo2015frequentist,Lin2017,Mucke2018}.
\begin{theorem}\label{consistency}
  Let $\A$ be KRR.
  Impose \ref{polyeigen}, \ref{eigenbound}, \ref{rass} and \ref{subgaussass}. Choose  
  \begin{equation}\label{rhodef}
    \rho = \c N^{- \frac{b}{2br' + 1}}
  \end{equation}
  for some $r' \le r$ and also let 
  \begin{equation}\label{ppp}
    P \le \c N^{\frac{2br' - 1}{2br' + 1}}.
  \end{equation} 
  Then for all positive $\x$ on a set of probability at least $1-\exp(-\x)$ we have 
  \begin{equation}\label{infbound}
    \infnorm{\fbar - \fstar}^2 \le \C \x^{5/2} \g^2 N^{-\frac{2br'-1}{2br' + 1}} + \C N^{- \frac{2br-1}{2br' + 1}}\norm{h}^2.
  \end{equation}
\end{theorem}
Also, under the same assumptions on the same set we have (as shown in \citep{avanesov2020})
\begin{equation}\label{l2cons}
  \norm{\fbar - \fstar} \le \C \x^{5/4} \g N^{-\frac{br'}{2br' + 1}} + \C N^{- \frac{br}{2br' + 1}} \norm{h}.
\end{equation}
Investigating \eqref{infbound} and \eqref{l2cons} makes it clear, that if we strove for a point estimate only it would be wise to choose $r' = r$ to guarantee optimal bias-variance trade-off (in both ${L}_2(\X,\pi)$-norm and sup-norm), yet our goal is inference and thus we have to abandon optimality for the sake of undersmoothness, choosing $r' < r$.

Finally, we are fully equipped to apply \ref{generaltheorem}. 
Straightforward algebra yields the following result in terms of O-tilde notation, hiding the factors poly-logarithmic in $P$ and $T$. 
\begin{corollary}\label{coroll}
  Let $\A$ be KRR. Impose \ref{lowersass}, \ref{polyeigen}, \ref{eigenbound}, \ref{rass} and \ref{subgaussass}. 
  Then for $N,P,T\rightarrow +\infty$
  \begin{equation}\label{pppp}
    \DP  = \tilde{O}\brac{N^{\frac{1}{2br'+1}}  \brac{P^{-1/6} + N^{-\frac{b(r-r')}{2br'+1}}} }.
  \end{equation}
  Moreover, under a further assumption $\varb \asymp 1$ we have
  \begin{equation}\label{ppppp}
    \DP = \tilde{O}\brac{P^{-1/6} + N^{-\frac{b(r-r')}{2br'+1}} }.
  \end{equation}
\end{corollary}
Clearly, the larger the gap between $r'$ and $r$, the further we stray from the optimal bias-variance trade-off thus reducing the accuracy of the point estimator. 
Under $\varb \asymp 1$, which is the case, for instance for a uniform measure $\pi$ and Gaussian noise $\varepsilon_i$, the gap may be arbitrarily small and still result in decay of the remainder term.
At the same time, if we abandon such a non-realistic assumption, as \eqref{pppp} dictates, the necessary discrepancy still becomes marginal for highly smooth functions and kernels ($b \gg 1$) as only $r' < r-1/b$ is required for the decay.

A question necessary to consider is whether a choice of $P$ and $r'$, both satisfying \eqref{ppp} and providing decay of the remainder term \eqref{pppp} exists, so we can have consistency and bootstrap validity at the same time.
Clearly, these requirements do not come into contradiction if $r' > 1/2+1/(2b)$ and such a choice is feasible if $r > 1/2 + 3/(2b)$. 
Again, this is only marginally more restrictive than the classical $\fstar \in \H$ for smooth functions ($b \gg 1$).
We have so far ignored how much accuracy of a point-estimate \emph{should} a practitioner sacrifice. 
In the light of \eqref{ppp}, the smallest reasonable choice of $r'$ is
\begin{equation}
  r' = \frac{3}{4} r + \frac{1}{8b}
\end{equation}
as for smaller $r'$ the first summands in \eqref{pppp} and \eqref{ppppp} dominate.
Substituting $r'$ to \eqref{infbound} and \eqref{l2cons} proves the choice to lead to an estimator exhibiting only marginally slower convergence in comparison to the choice $r'=r$ for large $b$.

\section{Experiment}
In this section we apply the method to the simulated data and estimate the actual probability for the yielded bands to cover $\ftrueT$.
We sample the covariates uniformly $X_i \sim U[0,1]$, $\tXt \sim U[0,1]$, the noise is Gaussian given the covariate, namely 
\begin{equation}
  \varepsilon_i \sim \Normal{0}{ e^{4 \abs{X_i - 0.5}} }
\end{equation}
and the true regression function $\fstar(x) = \sin(2 \cdot 3.14 x)$. The sample size $N$ is chosen to be $2^{16}$, the number of partitions $P \in \cbrac{2^6, 2^7, ..., 2^{12}}$ and the prediction set size $T \in \cbrac{2, 2^2, ..., 2^9}$. 
We choose the kernel $\kdd$ to be~\matern~with smoothness index $\nu=7/2$ (meaning $b=8$). 
The penalization parameter $\rho$ is chosen in accordance with \eqref{rhodef} with $r'=1/2$.
The experiment is repeated $2000$ times for each configuration and the estimated coverage probabilities are reported in Figure
\ref{fig1}. 
The nominal level is chosen $\alpha = 0.95$. We see that the probability for the method to yield the bounds which actually cover the true values of $\fstar$ does not depend on the size $T$ of the prediction set. 
Moreover, it meets the nominal level for $P > 128$ (in terms of 99\% confidence intervals reported in Table \ref{table} in Section \ref{tablesec}). 
Both of the findings are in line with the theoretical bootstrap validity result. 
As for the accuracy of the point-estimator $\fbar$, it declines only for $P\ge 2^{11}$ as such large numbers of partitions come into a contradiction with assumption \eqref{ppp}. 
Thus, for a wide range of $P$ and any practical size of the prediction set we have a proper coverage probability without a significant sacrifice in accuracy.

The code is available on GitHub.\footnote{\url{https://github.com/akopich/nonparregboot}}

\begin{figure}[h]
\centering\includegraphics[width=15cm]{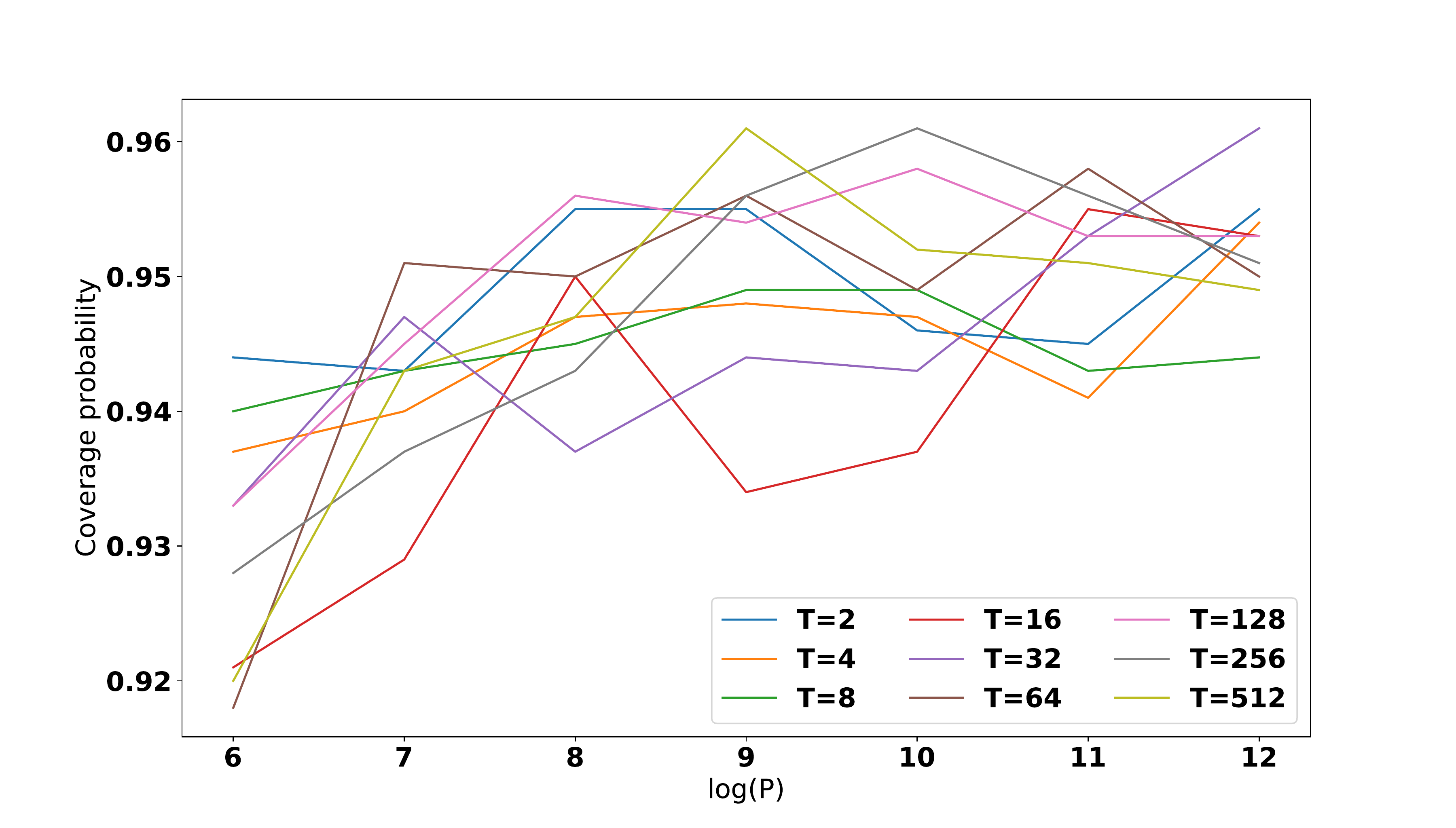}
\caption{Coverage probability in dependence on $P$ and $T$. The horizontal axis uses log-2 scale. The nominal coverage probability $\alpha=0.95$. Each point is averaged over 1000 trials.}\label{fig1}
\end{figure}

\bibliography{main}

\appendix
\section{Confidence intervals for the simulation study}\label{tablesec}

\begin{table}[H]
\begin{tabular}{l|lllllll}
   & $P=2^6$ & $P=2^7$ & $P=2^8$ & $P=2^9$ & $P=2^{10}$ & $P=2^{11}$ & $P=2^{12}$ \\ \hline
   $T=2^1$ & $(.92,.96)$ & $(.92,.96)$ & $(.94,.97)$ & $(.94,.97)$ & $(.92,.96)$ & $(.92,.96)$ & $(.94,.97)$ \\ 
   $T=2^2$ & $(.91,.95)$ & $(.92,.96)$ & $(.93,.96)$ & $(.93,.96)$ & $(.93,.96)$ & $(.92,.96)$ & $(.93,.97)$ \\ 
   $T=2^3$ & $(.92,.96)$ & $(.92,.96)$ & $(.92,.96)$ & $(.93,.96)$ & $(.93,.96)$ & $(.92,.96)$ & $(.92,.96)$ \\ 
   $T=2^4$ & $\boldsymbol{(.90,.94)}$ & $\boldsymbol{(.91,.95)}$ & $(.93,.97)$ & $(.91,.95)$ & $(.91,.95)$ & $(.94,.97)$ & $(.93,.97)$ \\ 
   $T=2^5$ & $(.91,.95)$ & $(.93,.96)$ & $(.91,.95)$ & $(.92,.96)$ & $(.92,.96)$ & $(.93,.97)$ & $(.94,.97)$ \\ 
   $T=2^6$ & $\boldsymbol{(.89,.94)}$ & $(.93,.97)$ & $(.93,.97)$ & $(.94,.97)$ & $(.93,.96)$ & $(.94,.97)$ & $(.93,.97)$ \\ 
   $T=2^7$ & $(.91,.95)$ & $(.92,.96)$ & $(.94,.97)$ & $(.93,.97)$ & $(.94,.97)$ & $(.93,.97)$ & $(.93,.97)$ \\ 
   $T=2^8$ & $\boldsymbol{(.90,.95)}$ & $(.91,.95)$ & $(.92,.96)$ & $(.94,.97)$ & $(.94,.97)$ & $(.94,.97)$ & $(.93,.97)$ \\ 
   $T=2^9$ & $\boldsymbol{(.90,.94)}$ & $(.92,.96)$ & $(.93,.96)$ & $(.94,.97)$ & $(.93,.97)$ & $(.93,.97)$ & $(.93,.96)$ \\ 
\end{tabular}
\caption{99\%-confidence intervals for the coverage probabilities obtained in the simulation study. The intervals not including the nominal $\alpha=0.95$ are shown in bold.}\label{table}
\end{table}

\section{Proof of Theorem 1}\label{secgenproof}
Throughout the section we impose \ref{varianceass}, \ref{biasass} and \ref{lowersass}. 
The first lemma implements \eqref{firststep}.
\begin{lemma}\label{gar1}
  \begin{equation}
    \sup_{l,u\in \R^T}\abs{\Prob{l < \brac{\fbarT-\ftruerhoT} \sqrt{P} < u} - \Prob{l < \gamma < u}} \le R_G,
  \end{equation}
  where 
  \begin{equation}
    R_G \coloneqq \C  \brac{\frac{{ \varb^{6}} \log^7(TP)}{P}}^{1/6} 
  \end{equation}
  and
  \begin{equation}
    \gamma \sim \Normal{0}{\Var{\fhatmT}}.
  \end{equation}
\end{lemma}
\begin{proof}
  Denote
  \begin{equation}
    Z_p \coloneqq \s^{-1}\brac{\fhatmT - \ftruerhoT}.
  \end{equation}
  Now, clearly, $\forall p \in [P]$ and $\forall t\in[T]$
  \begin{equation}
    \Var{\Zpt} \ge 1.
  \end{equation}
  Moreover, by \ref{varianceass} for all positive $\x$
  \begin{equation}
    \Prob{\abs{\Zpt} > \brac{1+\x}\varb } \le \exp(-\x).
  \end{equation}
  Now, using the integrated tail probability expectation formula, we have
  \begin{equation}
    \E{\abs{\Zpt}^3} \le \C \varb^{3},
  \end{equation}
  \begin{equation}
    \E{\abs{\Zpt}^4} \le \C \varb^{4}
  \end{equation}
  and
  \begin{equation}
    \E{\exp(\Zpt / \varb)} \le \C.
  \end{equation}
  Finally, we apply CLT (\ref{cltlemma}) which yields
  \begin{equation}
    \sup_{l,u\in \R^T}\abs{\Prob{l < P^{-1/2}\sump Z_p < u} - \Prob{l < \gamma < u}} \le R_G  ,
  \end{equation}
  where 
  $\gamma \sim \Normal{0}{\s^{-2}\Var{\fhatmT}}$.
\end{proof}
Now we demonstrate \eqref{app1}.
\begin{lemma}\label{gar}
  \begin{equation}
    \sup_{l,u\in \R^T}\abs{\Prob{l < \brac{\fbarT -\ftrueT}\sqrt{P} < u} - \Prob{l < \gamma < u}} \le R_G + \C \s^{-1}\sqrt{P}\B\brac{\sqrt{\log T}+1},
  \end{equation}
  where 
  \begin{equation}
    \gamma \sim \Normal{0}{\Var{\fhatmT}}.
  \end{equation}
\end{lemma}
\begin{proof}
  Re-writing the claim of \ref{gar1} yields for an arbitrary deterministic $a \in \R^T$
  \begin{equation}
    \sup_{l,u\in \R^T}\abs{\Prob{l < \brac{\fbarT - \ftruerhoT - a} \sqrt{P} < u} - \Prob{l+a\sqrt{P} < \gamma < u+a\sqrt{P}}} \le R_G.
  \end{equation}
  Next we choose $a = \ftruerhoT - \ftrueT$ and use \ref{biasass}, obtaining
  \begin{equation}
    \infnorm{a} \le \B.
  \end{equation}
  Further, we apply \ref{nazarov}, which bounds 
  \begin{equation}
    \sup_{l,u \in \R^T}\abs{\Prob{l+a\sqrt{P} < \gamma < u+a\sqrt{P}} - \Prob{l < \gamma < u}}  \le  \C \s^{-1}\sqrt{P}\B\brac{\sqrt{2\log 2T}+2}.
  \end{equation}
  The triangle inequality completes the proof.
\end{proof}
Now we establish a Gaussian approximation for the bootstrap estimator as the step \eqref{app2} suggests.
\begin{lemma}\label{garb}
  Let 
  \begin{equation}\label{plowerbound}
    \frac{\varb \log(TP) \log^{2}P }{P} \le 8^{-1}.
  \end{equation}
  Then
  \begin{equation}
    \sup_{l,u\in \R^T}\abs{\Probb{l < \brac{\fbarTb-\fbarT} \sqrt{P} < u} - \Probb{l < \gamma^\flat < u}} \le R_G^\flat + 4P^{-1},
  \end{equation}
  where 
  \begin{equation}
    R_G^\flat \coloneqq \C \brac{\frac{{ \varb^{6}}\log^7(TP)}{P}}^{1/6} \log P
  \end{equation}
  and
  \begin{equation}
    \gamma^\flat \sim \Normal{0}{\Varb{\fhatmTb}}.
  \end{equation}
\end{lemma}
\begin{proof}
  Here we will use the notation introduced in the proof of \ref{gar1}. Also define 
  \begin{equation}
    \Zbp \coloneqq \s^{-1}\brac{\fhatmTb - \fbarT}.
  \end{equation}
  For all positive $\x$ on a set of probability at least $1-P\exp(-\x)$ for all $p \in [P]$
  \begin{equation}
    \abs{\Zpt} < \brac{1+\x}\varb.
  \end{equation}
  Thus, on this set for all $t \in [T]$ and $p \in [P]$
  \begin{equation}
    \Eb{\abs{\Zpt^\flat}^3} \le \C \brac{\brac{1+\x} \varb}^{3},
  \end{equation}
  \begin{equation}
    \E{\abs{\Zpt^\flat}^4} \le \C \brac{\brac{1+\x} \varb}^{4}
  \end{equation}
  and
  \begin{equation}
    \E{\exp \brac{\frac{\Zpt^\flat}{(1+\x)\varb}}  } \le \C.
  \end{equation}
  Still conditioning on this set we can apply Hoeffding's inequality. So, for all $\u>0$ on a set of probability at least $1-2T\exp(-\u)$ for all $t\in[T]$
  \begin{equation}
    \abs{ \frac{1}{P} \sump \brac{\brac{\fhatmT - \ftruerhoT}_t^2 - \Var{\brac{\fhatmT}_t} } } > (1+\x)^2\VS^2 \sqrt{\frac{\u}{2P}}.
  \end{equation}
  Now choose $\x=2\log P$, $\u = \log(TP)$ and notice, that under \eqref{plowerbound} 
  \begin{equation}
    \min_{t \in [T]} \Varb{\Zpt^\flat} \ge \s^2/2.
  \end{equation}
  Finally, we are fully equipped to apply \ref{cltlemma} and account for conditioning (\ref{condlemma}).
\end{proof}

Denote a diagonal matrix composed of diagonal elements of $\Var{\fhatmT}$ as 
\begin{equation}
  \Lambda^{-2} = diag \brac{\Var{\fhatmT}_{11}, \Var{\fhatmT}_{22}, ..., \Var{\fhatmT}_{TT}}.
\end{equation}
The last but one ingredient is a concentration of the covariance matrix of $\fbarTb$ around the covariance matrix of its real-world counterpart $\fbarT$.
\begin{lemma}\label{covcov1}
  With probability at least $1-{P^{-10}}$
  \begin{equation}
    \infnorm{\Var{\Lambda\fhatmT} - \Varb{\Lambda\fhatmTb}} \le \C \frac{ {\brac{\infnorm{\fstar}^2 + \varb^2}  \log^{5/2} \brac{ P \vee T} }}{\sqrt{P}}.
  \end{equation}
\end{lemma}
\begin{proof}
   First, consider 
  \begin{equation}
    \ftT \coloneqq \ftruerhoT + \Lambda\brac{\ftruerhoT - \fhatmT}
  \end{equation}
  and in the same way
  \begin{equation}
    \ftTb \coloneqq \fbarT + \Lambda\brac{\fbarT - \fhatmTb}.
  \end{equation}
  Clearly, 
  \begin{equation}
    \Var{\ftT } = \Var{\Lambda\fhatmT} \text{ and } \Varb{\ftTb } = \Varb{\Lambda\fhatmTb}.
  \end{equation}
  Now notice, $\Varb{\ftTb}$ is actually an empirical covariance matrix
  \begin{equation}
    \Varb{\ftTb } = \frac{1}{P} \sump (\ftT - \fbarT)(\ftT - \fbarT)^\top,
  \end{equation}
  which boils the problem down to demonstrating the concentration of an empirical covariance matrix. We shall achieve that via applying Hoeffding's inequality element-wise.
  By \ref{varianceass} with probability at least $1-P\exp(-\x)$ for all $p \in [P]$
  \begin{equation}
    \infnorm{\ftT \ftT^\top} \vee \infnorm{\ftT \fbarT^\top} \vee \infnorm{\fbarT \fbarT^\top} \le \C\brac{\infnorm{\fstar} + \brac{1+\x}\s^{-1}\VS}^2
  \end{equation} 
  and hence
  \begin{equation}
    \infnorm{ (\ftT - \fbarT)(\ftT - \fbarT)^\top} \le \C\brac{\infnorm{\fstar} + \brac{1+\x}\s^{-1}\VS}^2.
  \end{equation}
  Thus, by Hoeffding's inequality we have for any positive $\u$
  \begin{equation}
    \Prob{\infnorm{\Var{\ftT} - \Varb{\ftTb}} > \C \frac{\u \brac{\infnorm{\fstar} + \brac{1+\x}\varb}^2  }{\sqrt{P}}} \le T^2 \exp(-\u^2) + P\exp(-\x).
  \end{equation}
  The choice $\u^2 = 10 \log P + 2 \log T$, and $\x = 11 \log P$ completes the argument.
\end{proof}
Finally, we are ready to consummate the final step \eqref{app3}.
\begin{lemma}\label{covcovcov}
  Consider two independent centered Gaussian vectors $X, Y\in \R^T$ with covariance matrices $\Var{\fhatmT}$ and $\Var{\fhatmTb}$ respectively. Let $\s^2 = \min_i \Var{X_i}$.
  Then uniformly over $l,u\in \R^T$ we have
  \begin{equation}
    \abs{\Prob{l < X < u} - \Prob{l < Y < u}} \le \C\Delta^{1/3}\brac{1 \vee \log (TP)}^{2/3} + 2{P^{-10}},
  \end{equation}
  where 
  \begin{equation}
    \Delta \coloneqq  \C \frac{ { \brac{\infnorm{\fstar}^2 + \varb^2}  \log^{5/2} \brac{ P \vee T} }}{\sqrt{P}}.
  \end{equation}
\end{lemma}
\begin{proof}
  We  cannot apply \ref{comparison1} directly, as the variances of the components of $X$ are not constant (they depend on $S$). To mitigate the issue we apply it to vectors $X' = \Lambda X$ and $Y' = \Lambda Y$ instead, employing \ref{covcov1}. The rest is due to \ref{condlemma}, accounting for conditioning  and the fact that $\Delta \ge P^{-1}$.
\end{proof}

\begin{proof}{\bf{of \ref{generaltheorem}.}}
  The argument consists in combining of \ref{gar}, \ref{garb} and \ref{covcovcov} by the means of the triangle inequality. 
  Notice, $P^{-1}$ gets absorbed by the slower decaying terms, which yields
  \begin{equation}
    \begin{split}
      \DP &\le \C \brac{P^{-1} \log^{7}(TP)}^{1/6} \varb \log P \\
      & + \C \brac{P^{-1} \brac{\infnorm{\fstar}^4 + \varb^4} }^{1/6} \log^{5/6}\brac{P \vee T} \log^{2/3} (PT) \\
      & + \C \s^{-1} \sqrt{P}\B \brac{\sqrt{\log T} +1}.
    \end{split}
  \end{equation}
  Finally, recall that $\varb \ge \c$ and thus $\varb^6 \ge \C\varb^4$, which allows us  to ``combine'' the first two lines.
\end{proof}
\section{Proofs for KRR}\label{KRRproofsec}
We open the section with some notation.
Consider the eigenvalues $\cbrac{\mu_j}_{j=1}^\infty$ and normalized eigenfunctions $\cbrac{\phi_j(\cdot)}_{j=1}^\infty$ of $k(\cdot, \cdot)$ w.r.t. the continuous measure $\pi$.
Now for $f\in \H$ we have the vector $\param$ of scaled expansion coefficients $\param_j \coloneqq \mu_j^{-1/2} \dotprod{f}{\phi_j}$. 
Further, we introduce the design matrix $\Psi \in \R^{n\times \infty}$ s.t. $\Psi_{ij} = \sqrt{\mu_j}\phi_j(X_i)$. 
This allows for the following reformulation of the problem \eqref{krr} 
\begin{equation}
  \hatparam \coloneqq \arg\max_{\param}\brac{\overbrace{\underbrace{-\frac{1}{2n}\normSq{\Psi \param - y}}_{L(\param)}-\frac{\rho}{2}\normSq{\param}}^{\Lrho(\param)}}.
\end{equation}
We also define 
\begin{equation}
  \trueparam \coloneqq \arg\max_{\param} \E{L(\param)}
\end{equation}
and its penalized counterpart
\begin{equation}\label{thetastarrhodef}
  \trueparamrho \coloneqq \arg\max_{\param} \E{L(\param)} - \frac{\rho}{2}\normSq{\param}.
\end{equation}
We also introduce a vector $\eps \in \R^n$ s.t. $\eps_i = \varepsilon_i$.
By the means of trivial calculus we have 
\begin{equation}
  \score \coloneqq \nabla\brac{L(\param) - \mathbb{E}_{\varepsilon}[L(\param)]} = \frac{1}{n} \Psit\eps
\end{equation}
and 
\begin{equation}
  \DrhoSq \coloneqq -\nabla^2\brac{L(\param)-  2\rho\norm{\param}} = \frac{1}{n} \Psit \Psi + \rho I.
\end{equation}
Taylor expansion around the stationary point $\hatparam$ yields
\begin{equation}
  \nabla\Lrho(\param)  =  -\DrhoSq \brac{\param - \hatparam}.
\end{equation}
Next, we notice $\nabla \E{\Lrho(\trueparamrho)} = 0$ and have
\begin{equation}
  \nabla \Lrho(\trueparamrho) = \nabla \zeta = -\DrhoSq \brac{\trueparamrho - \hatparam}.
\end{equation}
Finally, relying on the fact that $\rho>0$, thus $\DrhoSq$ is invertible, we multiply both sides of the last equation by $\DrhoInv$, obtaining {\it Fisher expansion}
\begin{equation}\label{fisher}
  \Drho \brac{\hatparam - \trueparamrho} = \DrhoInv \nabla \zeta.
\end{equation}
This result has been generalized in \citep{Spokoiny2019} (Theorem 2.2).

Now consider a diagonal matrix $M \in \R^{\infty\times\infty}$ s.t. $M_{jj}^2={\mu_j}$ and notice $M^2 + \rho I = \E{\DrhoSq}$. This motivates the following design regularity assumption.
\begin{assumption}[Design regularity]\label{designass}
  Let there exist some positive $\delta$ s. t.
  \begin{equation}
  \norm{\brac{M^2 + \rho I}\minusroot \DrhoSq \brac{M^2 + \rho I}\minusroot - I} \le \delta < 1.
  \end{equation}
\end{assumption}
\ref{lemmadesign} (proven in \citep{avanesov2020}) guarantees that it holds for a random design with $X_i$ being i.i.d. and distributed w.r.t. a continuous measure.

\begin{lemma}\label{tracetrick}
\begin{equation}
  \sumj \frac{\mu_j}{(\mu_j  + \rho)^2} \le \frac{\C\trace}{\rho}.
\end{equation}
\end{lemma}
\begin{proof}
  \begin{equation}
  \begin{split}
    \sumj \frac{\mu_j}{(\mu_j  + \rho)^2} \le \frac{1}{\rho} \sumj \frac{\mu_j}{\mu_j + \rho } 
     = \frac{\C \trace}{\rho}.
  \end{split}
  \end{equation}
\end{proof}
In the rest of the section we implicitly impose  \ref{polyeigen}, \ref{eigenbound}, and \ref{subgaussass}.
The next lemma bounds the variance term in RKHS-norm. 
\begin{lemma}\label{varianceHK}
  For all $\x>0$, $\t>2.6$
  \begin{equation}
    \Prob{\HknormSq{\frho-\fhat} \le \C\x\t \g^2  \frac{\trace}{n\rho} } \ge 1-e^{-\x} - e^{-\t/2}.
   \end{equation}
\end{lemma}
\begin{proof}
  By construction, using Fisher expansion \eqref{fisher} and under \ref{designass} we have 
  \begin{equation}
  \begin{split}
    \HknormSq{\frho-\fhat} &= \normSq{\trueparamrho - \hatparam} \\
    &= \normSq{\DrhoSqInv \nabla \zeta} \\
    &= \frac{1}{n^2} \epst\Psit \brac{\frac{1}{n}\Psit\Psi + \rho I}^{-2} \Psi \eps \\
    &\le \frac{1}{(1-\delta)} \epst\underbrace{\Psit \frac{1}{n^2}\brac{M^2 + \rho I}^{-2} \Psi}_A \eps.
  \end{split}
  \end{equation}
  For every diagonal element of $A$ we have by \ref{tracetrick}
  \begin{equation}
  \begin{split}
    A_{ii} &\le \frac{\phiinf^2}{n^2}\sumj \frac{\mu_j}{(\mu_j  + \rho)^2} \\
    & \le \frac{\C \trace}{n^2\rho}.
  \end{split}
  \end{equation}
  Thus, trivially
  \begin{equation}
    \tr{A} \le \frac{\C \trace}{n\rho},
  \end{equation}
  
  \begin{equation}
    \tr{A^2} \le  \brac{\frac{\C \trace}{n\rho}}^2
  \end{equation}
  and
  \begin{equation}
    \normF{A}=\sqrt{\tr{A^\top A}} \le  \frac{\C \trace}{n\rho}.
  \end{equation}
  Applying the Hanson-Wright inequality \citep{rudelson2013} followed by applying \ref{lemmadesign} finalizes the proof.
\end{proof} 

\begin{lemma}\label{variancebarHK}
  Let 
  \begin{equation}\label{pbound}
    P \le \c N\rho^{2/b}.
  \end{equation}
  Then for all $\x>0$
  \begin{equation}
    \Prob{\HknormSq{\frho-\fbar} \le \C\x^{5/2} \g^2  \frac{\trace}{n\rho} } \ge 1-e^{-\x}.
   \end{equation}
\end{lemma}
\begin{proof}
  Via Fisher's expansion \eqref{fisher} we have for each $p \in [P]$
  \begin{equation}
    \hatparamm - \trueparamrho = \DrhomSqInv \sum_{i \in S_p} \psi(X_i) \eps_i,
  \end{equation}
  under \ref{designass} we average over $p\in [P]$
  \begin{equation}
    \norm{\barparam - \trueparamrho} \le \norm{\frac{1}{1-\delta} \brac{M^2 + \rho I}^{-1} \Psi \eps}.
  \end{equation}
  Now by the same argument as in \ref{varianceHK} and using \eqref{pbound} (providing $\brac{1 - \delta}^{-1} < \c \t$) we obtain the claim after the choice $\x = \t$.
\end{proof}

\begin{lemma}\label{varianceinf}
  \ref{varianceass} holds for Kernel Ridge Regression with 
  \begin{equation}
    \Vn =\C \sqrt{\frac{\trace}{n \rho^{1/b}}}.
  \end{equation}
  Moreover, under \ref{polyeigen}
  \begin{equation}
    \Vn \le  {\frac{\C}{\sqrt{n} \rho^{1/b}}}.
  \end{equation}
\end{lemma}
\begin{proof}
  For $L_2(\X, \pi)$-norm it is established in \citep{avanesov2020} that 
  for all $\x>0$, $\t>2.6$
  \begin{equation}
    \Prob{\normSq{\frho-\fhat} \le \C\x\t \g^2  \frac{\trace}{n} } \ge 1-e^{-\x} - e^{-\t/2}.
   \end{equation}
   Combining this bound with the claim of \ref{varianceHK} by the means of \ref{mendelson} yields the first line of the claim. The second line follows from the well-known fact  
   \begin{equation}
     \trace \le \C \rho^{-1/b}.
   \end{equation}
   See any of \citep{Zhang2015, Yang2017a, Spokoiny2019,avanesov2020} for the proof.
\end{proof}

\begin{proof}{\bf{of \ref{consistency}}.}
  The bound for RKHS-norm follows from \ref{variancebarHK} and \ref{biaslemma} via the triangle inequality. The sup-norm bound follows from combining the bounds in $\Ltwo$- and RKHS-norm by the means of \ref{mendelson}.
\end{proof}

\section{Tools}
This section briefly cites the results we relied upon. First of all, we need a central limit theorem.
\begin{lemma}[Corollary of the main theorem by \cite{Chernozhukov2017}]\label{cltlemma}
  Consider centered i.i.d. $Z_1$, $Z_2$, ..., $Z_n\in \R^T$ with a covariance matrix $\Sigma$. 
  Let there be $B_n>0$ s.t. $\forall j\in[T]\text{, } \forall k \in \cbrac{1,2}$
  \begin{equation}
    \E{\abs{Z_{i,j}}^{2+k}} \le B_n^k,
  \end{equation}
  \begin{equation}
    \E{\exp\brac{\abs{Z_{i,j}}/B_n}} \le 2
  \end{equation}
  and 
  \begin{equation}
    \E{Z_{i,j}^{2}} > \c.
  \end{equation}
  Then for $\gamma \sim \Normal{0}{\Sigma}$
  \begin{equation}
    \sup_{l,u \in \R^T}\abs{\Prob{l < n^{-1/2} \sum_{i=1}^n Z_i < u} - \Prob{l < \gamma < u}} \le \C \brac{\frac{B_n^2 \log^7(Tn)}{n}}^{1/6}.
  \end{equation}
\end{lemma}

The next lemma may be obtained by applying the modification employed in the proof of Theorem 4.1 in \citep{Chernozhukov2017} to the proof of Theorem 2 in \citep{comparison}.
\begin{lemma}\label{comparison1}
  Consider a pair of independent centered Gaussian vectors $X,Y \in \R^T $ with covariance matrices $\Sigma_1$ and $\Sigma_2$ respectively. Let 
  \begin{equation}
    \infnorm{\Sigma_1 - \Sigma_2} \le \delta.
  \end{equation}
  Then uniformly over vectors $l, u \in \R^T$ we have
  \begin{equation}
    \abs{\Prob{l < X < u} - \Prob{l < Y < u}} \le B\delta^{1/3}\max\cbrac{1, \log (T/\delta)}^{2/3},
  \end{equation}
  where $B$ depends only on $\max_{i}\Sigma_{1,ii}$ and $ \min_{i} \Sigma_{1,ii}$.
\end{lemma}

Now we cite the anti-concentration for multivariate Gaussian vectors.
\begin{lemma}[Nazarov’s inequality \citep{nazarov}, Theorem 1 \citep{Chernozhukov2017}]\label{nazarov}
  Consider a centered $p$-dimensional Gaussian vector $X \sim \Normal{0}{\Sigma}$. Then for any deterministic $\delta>0$
  \begin{equation}
    \sup_{y\in\R^p} \brac{\Prob{X < y + \delta} - \Prob{X < y}} \le \frac{\delta}{\sqrt{\min_{i}\Sigma_{ii}}}\brac{\sqrt{2\log p}+2}.
  \end{equation}
\end{lemma}
We also need to control the bias of KRR with respect to $L_2(\X,\pi)$-norm and RKHS-norm.
\begin{lemma}[Theorem 4 \citep{smale04}]\label{biaslemma}
  Let \ref{rass} hold for $\fstar$ with $r\in (1/2,1]$ and $h\in \Ltwo$. Then 
  \begin{equation}
    \norm{\ftruerho - \fstar} \le \rho^{r} \norm{h}
  \end{equation}
  and
  \begin{equation}
    \Hknorm{\ftruerho - \fstar} \le \rho^{r-1/2} \norm{h}.
  \end{equation}
\end{lemma}
Next we cite the result connecting the sup-norm, $L_2(\X,\pi)$-norm and RKHS-norm.
\begin{lemma}[Lemma 5.1 \citep{Mendelson2010}]\label{mendelson}
  Under \ref{polyeigen} and \ref{eigenbound} for all $f \in \H$
  \begin{equation}
    \infnorm{f} \le \C \Hknorm{f}^{1/b}\norm{f}^{1-1/b}.
  \end{equation}
\end{lemma}
Combining \ref{mendelson} and \ref{biaslemma} yields
\begin{lemma}\label{biaslemmainf}
  Under \ref{polyeigen} and \ref{eigenbound} let \ref{rass} hold for $\fstar$ with $r\in(1/2,1]$ and $h\in \Ltwo$. Then 
  \begin{equation}
    \infnorm{\fstar - \ftruerho} \le \C \rho^{r-1/(2b)}\norm{h}.
  \end{equation}
\end{lemma}
The next trivial lemma quantifies how much a measure of a set changes after conditioning. 
\begin{lemma}\label{condlemma}
  Consider a measure $\mathbb{P}$ and two measurable sets $A$ and $B$. Then 
  \begin{equation}
    \abs{\Prob{A} - \Prob{A|B}} \le 2 \Prob{\bar B}.
  \end{equation}
\end{lemma}
\begin{proof}
  \begin{equation}
  \begin{split}
    \abs{\Prob{A} - \Prob{A|B}} &= \abs{\Prob{A|B}\Prob{B} + \Prob{A|\bar B}\Prob{\bar B} - \Prob{A|B} } \\
    &= \abs{\Prob{A|B}(\Prob{B}-1) + \Prob{A|\bar B}\Prob{\bar B} } \\
    &\le 2\Prob{\bar B}.
  \end{split}
  \end{equation}
\end{proof}
Further we cite a result which provides sufficient conditions for high probability of \ref{designass} to hold.
\begin{lemma}[Lemma 10 by \cite{avanesov2020}]\label{lemmadesign}
  Let $\cbrac{\phi_j(\cdot)}_{j=1}^\infty$ and $\cbrac{\mu_j}_{j=1}^\infty$ be eigenfunctions and eigenvalues of $k(\cdot,\cdot)$ w.r.t. $\pi$. 
  Let \ref{eigenbound} hold.
  Then \ref{designass} holds for some $\C > 0$ and arbitrary $\t>2.6$ with 
  \begin{equation}
    \delta =  \C \trace\brac{\sqrt{\frac{\t}{n}} + \frac{\t}{n}}
  \end{equation}
  on a set of probability at least $1-e^{-\t/2}$.
\end{lemma}

\let\c =\oldc    

\end{document}